%% file: improvment_pbda.tex
\documentclass{article} 
\usepackage{nips14submit_e,times}

\usepackage{amssymb}
\usepackage{latexsym}
\usepackage{amsbsy}
\usepackage{amsmath}
\usepackage{amsthm}
\usepackage{enumerate}
\usepackage{graphicx}
\usepackage{subfigure}
\usepackage[numbers]{natbib}
\usepackage{xspace}

\include{entete}

\newtheorem{definition}{Definition}

\newtheorem{theorem}{Theorem}
\newtheorem{proposition}[theorem]{Proposition}

\title{An Improvement to the Domain Adaptation Bound\newline in a PAC-Bayesian context}

\author{
Pascal Germain\\
IFT-GLO, Universit\'e Laval\\
Qu\'ebec (QC), Canada \\
\texttt{\small pascal.germain@ift.ulaval.ca}
\And Amaury Habrard \\
LaHC, UMR CNRS 5516\\
Univ. of St-Etienne, France\\
\texttt{\small amaury.habrard@univ-st-etienne.fr}
\And  Fran\c cois Laviolette \\
IFT-GLO, Universit\'e Laval\\
Qu\'ebec (QC), Canada \\
\texttt{\small francois.laviolette@ift.ulaval.ca}
\And Emilie Morvant \\
LaHC, UMR CNRS 5516\\
Univ. of St-Etienne, France\\
\texttt{\small emilie.morvant@univ-st-etienne.fr}
}

\nipsfinalcopy 

\renewcommand{\PS}{P_s}
\renewcommand{\PT}{P_t}
\renewcommand{\DS}{D_s}
\renewcommand{\DT}{D_t}
\renewcommand{\ms}{m_s}
\renewcommand{\xbfs}{\xbf_i}
\renewcommand{\ys}{y_i}
\renewcommand{\xbft}{\xbf_j}

\begin{document}

\maketitle

\begin{abstract}
This paper provides a theoretical analysis of domain adaptation based
on the PAC-Bayesian theory. We propose an improvement of the previous
domain adaptation bound obtained by Germain et al. \cite{PBDA} in two
ways. We first give another generalization bound tighter and
easier to interpret. Moreover, we provide a new analysis of the
constant term appearing in the bound that can be of high interest for
developing new algorithmic solutions.
\end{abstract}

\section{Introduction}
Domain adaptation (DA) arises when the distribution generating the target
data differs from the one from which the source learning has been
generated from. Classical theoretical analyses of domain adaptation
propose some generalization bounds over the expected risk of a 
classifier belonging to a hypothesis class $\Hcal$   over the
target domain  \cite{BenDavid-NIPS06,BenDavid-MLJ2010,Mansour-COLT09}. Recently, Germain et al. have
given a generalization bound  expressed as an averaging over  the
classifiers in $\Hcal$ using the PAC-Bayesian theory \cite{PBDA}. In this
paper, we derive a new PAC-Bayesian domain adaptation bound that
improves the previous result of \cite{PBDA}. Moreover, we provide an analysis of the
constant term appearing in the bound opening the door to design 
new algorithms able to control this term. 
The paper is organized as follows. We introduce the classical
PAC-Bayesian theory in Section~\ref{s:PAC-B}. We present the domain adaptation
bound obtained in \cite{PBDA} in Section~\ref{pbda}. Section~\ref{s:new} presents our new
results.

\section{PAC-Bayesian Setting in Supervised Learning}
\label{s:PAC-B}
 In the non adaptive setting, the PAC-Bayesian theory \cite{McAllester99b} offers generalization bounds (and algorithms) for weighted majority votes over a set of functions, called voters.
Let $X \subseteq  \R^d$ be the input space of dimension $d$ and $Y = \{-1,+1\}$ be the output space. 
A domain $\PS$ is an unknown distribution over $X\times Y$.  
The marginal distribution of $\PS$  over $X$ is denoted by $\DS$. 
Let $\Hcal$ be a set of $n$  voters such that: $\forall h  \in  \Hcal, \ h : X \to  Y$, and let 
 $\prior$ be  a prior  on $\Hcal$.  A \emph{prior} is a probability distribution on $\Hcal$ that ``models'' some {\it a priori} knowledge on quality of the voters of $\Hcal$.
 
 Then, 
given a learning sample $S = \{(\xbfs ,\ys)\}_{i=1}^{\ms}$\,,  drawn independently and identically distributed  ({\it i.i.d.}) according to the distribution $\PS$, 
the aim of the PAC-Bayesian learner is to find a posterior distribution $\posterior$ leading to a $\posterior$-weighted majority vote $\BQ$ over $\Hcal$ that has the lowest possible expected risk, \emph{i.e.,} the lowest probability of making an error on future examples drawn from $\DS$. More precisely, the vote $\BQ$ and its true and empirical risks are defined as follows.
\begin{definition}\rm
Let $\Hcal$ be a set of  voters. Let $\posterior$ be a distribution over $\Hcal$. The $\posterior$-weighted majority vote~$\BQ$ (sometimes called the Bayes classifier) is:
\begin{align*}
\forall \xbf\in X,\quad \BQ(\xbf) \ \eqdef\ \sign\left[\esp{h\sim \posterior} h(\xbf)\right].
\end{align*}
The true risk of $\BQ$ on a domain $\PS$ and its empirical risk on a $\ms$-sample $S$ are respectively:
\begin{equation*}
\RPS(\BQ) \ 
\eqdef\esp{(\xbfs,\ys)\sim \PS}  \I\left[\BQ(\xbfs)\ne \ys\right],
 \quad \mbox{ and }\quad \RS(\BQ) 
\ \eqdef \  \frac{1}{\ms} \ \sum_{i=1}^{\ms}  \I\left[\BQ(\xbfs)\ne \ys\right].
\end{equation*}
\end{definition}
where $\I[a\neq b]$ is the 0-1 loss function returning $1$ if $a=b$ and 0 otherwise.
Usual PAC-Bayesian analyses
\cite{McAllester99b,McAllester03,Seeger02,catoni2007pac,germain2009pac} do not
directly focus on the risk of $\BQ$, but bound the risk of the closely
related stochastic Gibbs classifier $\GQ$. It predicts the label of an
example $\xbf$ by first drawing a classifier $h$ from $\Hcal$
according to $\posterior$, and then it returns $h(\xbf)$. Thus, the
true risk and the empirical on a $\ms$-sample $S$  of $\GQ$ correspond to the expectation of the risks over $\Hcal$ according to $\posterior$:
\begin{eqnarray*}
\RPS(\GQ) &\eqdef& \esp{h\sim \posterior} \RPS(h) 
\ = \esp{(\xbfs,\ys)\sim \PS} \esp{h\sim \posterior}  \I\left[h(\xbfs)\ne \ys\right], \,
\\
\mbox{ and }\, \RS(\GQ)  &\eqdef&  \esp{h\sim \posterior} \RS(h) \ = \ \frac{1}{\ms} \   \sum_{i=1}^{\ms}  \esp{h\sim\posterior} \I\left[h(\xbfs)\ne \ys\right]. \label{eq:gibbs}
\end{eqnarray*}
Note that it is well-known in the PAC-Bayesian literature that the risk of the deterministic classifier~$\BQ$ and the risk of the stochastic classifier $\GQ$ are related by $\RPS(\BQ)\leq 2\RPS(\GQ)$.

\section{PAC-Bayesian Domain Adaptation of the Gibbs classifier}
\label{pbda}
Throughout the rest of this paper, we consider the PAC-Bayesian DA setting introduced by Germain et al. \cite{PBDA}.
The main difference between supervised learning and DA is that we have two different domains over $X\times Y$: the source domain $\PS$ and the target domain $\PT$ ($\DS$ and $\DT$ are the respective marginals over $X$).
The aim is then to learn a good model on the target domain $\PT$ knowing that we only have label information from the source domain $\PS$. 
Concretely, in the setting described in \cite{PBDA}, we have a labeled source {sample} \mbox{$S=\{(\xbfs,\ys)\}_{j=1}^{\ms}$}\,, drawn {\it i.i.d.} from $\PS$ and a target unlabeled {sample} $T=\{\xbft\}_{j=1}^{\mt}$\,, drawn {\it i.i.d.} from $\DT$.
One thus desires  to learn from $S$ and $T$ a  weighted majority vote with the lowest possible expected risk on the target domain $\RPT(\BQ)$, {\it i.e.}, with good generalization guarantees on $\PT$.
Recalling that usual PAC-Bayesian generalization bound study the risk of the Gibbs classifier, Germain et al.~\cite{PBDA} have done an analysis of its target risk $\RPT(\GQ)$, which also relies on the notion of \emph{disagreement} between the voters:
\begin{equation} \label{eq:disagreement}
R_D(h,h')\ \eqdef  \ \esp{\xbf\sim D} \I[h(\xbf)\ne h'(\xbf)]\,.
\end{equation}
Their main result is  the following theorem.
\begin{theorem}[{\small Theorem~4 of \citep{PBDA}}]
\label{thm:pbda}
Let $\Hcal$ be a set of voters. For every distribution $\posterior$ over $\Hcal$, we have:
\begin{align}\label{eq:pbda}
\RPT(\GQ)\ \leq\ \RPS(\GQ) + \des(\DS,\DT) + \lambda_{\posterior,\posterior^*_T}\,,
\end{align}
where $\des(\DS,\DT)$ is the domain disagreement between the marginals $\DS$ and $\DT$,
\begin{equation}
\label{eq:dis}
\des(\DS,\DT)\ \eqdef\ \left|\esp{(h,h')\sim\posterior^2}\!  \!    \left(  \RDS(h,h') - \RDT(h,h')  \right)  \right|,
\end{equation}
 with $\posterior^2(h,h') = \posterior(h)\! \times\! \posterior(h')$\,, 
and  
$\lambda_{\posterior,\posterior^*_T} = \RPT(G_{\posterior^*_T}) + \RDT(\GQ,G_{\posterior^*_T}) + \RDS(\GQ,G_{\posterior^*_T})$\,, 
where~$\posterior^*_T = \argmindevant{\posterior} \RPT(\GQ)$ is the best distribution on the target domain.
\end{theorem}
Note that this bound reflects the usual philosophy in DA: It is well known  that a favorable situation for DA arrives when the divergence between the domains is small while achieving good source performance \cite{BenDavid-NIPS06,BenDavid-MLJ2010,Mansour-COLT09}. 
Germain et al.~\cite{PBDA} have then derived a first promising algorithm called PBDA for minimizing this trade-off between source risk and domain disagreement.\\
Note that Germain et al.~\cite{PBDA} also showed that, for a given hypothesis class $\Hcal$, the domain disagreement of Equation~\eqref{eq:dis} is always smaller than the $\hdh$-distance of Ben-David et al.~\cite{BenDavid-NIPS06,BenDavid-MLJ2010} defined by $\tfrac{1}{2} \sup_{(h,h')\in\Hcal^2} \left|\RDT(h,h')-\RDS(h,h')\right|$.

\section{New Results}
\label{s:new}
\subsection{Improvement of Theorem~\ref{thm:pbda}}

First, we introduce the notion of {\it expected joint error} of a pair of classifiers $(h,h')$ drawn according to the distribution $\posterior$, defined as
\begin{equation}
\label{eq:eP}
\ePS(\GQ,\GQ) \  \eqdef\,   
\esp{(h,h') \sim\posterior^2}\esp{(\xbf,y) \sim \PS} \I [ h(\xb) \neq y ] \times \I [ h'(\xb) \neq y ]\,.
\end{equation}

Thm~\ref{thm:pacbayesdabound} below relies on the domain disagreement of Eq.~\eqref{eq:disagreement}, and  on  \emph{expected  joint error} of Eq.~\eqref{eq:eP}.

\begin{theorem}
\label{thm:pacbayesdabound}
Let ${\cal H}$ be a hypothesis class. We have 
\begin{align} \label{eq:new}
\forall \posterior&\mbox{ on }\Hcal,\quad  \RPT(G_\posterior)\ \leq \  \RPS(G_\posterior) +  \frac{1}{2}\des(\DS,\DT) + \lambda_\posterior\,, 
\end{align}
where $\lambda_\rho$  
is the deviation between the expected joint errors of $G_\posterior$ on the target and source domains:
 \begin{eqnarray} \label{eq:lambda_rho}
 \nonumber
 \lambda_\posterior\!&\eqdef& 
 \Big|\, \ePT(G_\posterior, G_\posterior) - \ePS(G_\posterior,G_\posterior) \,\Big|\,.
\end{eqnarray}
\end{theorem}
\begin{proof}
First, note that for any distribution $P$ on $X\times Y$, with marginal distribution $D$ on $X$, we have
\begin{equation*} \label{eq:rde}
R_P(G_\posterior) \ = \ \frac{1}{2} \RD(G_\posterior,\GQ) + \eP(G_\posterior,G_\posterior)\,,
\end{equation*}
\begin{eqnarray*}
\mbox{as\ \ \ \ \ \ \  }2\, R_P(G_\posterior)
&=& 
\esp{(h,h') \sim\posterior^2}\esp{(\xbf,y) \sim P} 
\Big( \I [ h(\xb) \neq y ] + \I [ h'(\xb) \neq y ]  \Big) \\
&=& 
\esp{(h,h') \sim\posterior^2}\esp{(\xbf,y) \sim P} 
\Big( 1\times \I [h(\xb)\neq h'(\xb) ] + 2\times \I [h(\xb)\neq y] \, \I [h'(\xb)\neq y]\Big)\\[1.5mm]
&=& 
\RD(G_\posterior,\GQ) + 2\times \eP(G_\posterior,G_\posterior)\,.
\end{eqnarray*}
Therefore,
\begin{align*}
\nonumber \RPT(G_\posterior)-\RPS(G_\posterior)
& = 
\nonumber \frac{1}{2} \Big(\RDT(G_\posterior,\GQ)-\RDS(G_\posterior,\GQ)\Big) \!+\!\Big(\ePT(G_\posterior,G_\posterior)-\ePS(G_\posterior,G_\posterior)\Big) \\
&\leq
\nonumber \frac{1}{2} \Big|\RDT(G_\posterior,\GQ)-\RDS(G_\posterior,\GQ)\Big| +\Big|\ePT(G_\posterior, G_\posterior)-\ePS(G_\posterior, G_\posterior)\Big|  \\
&=
\frac{1}{2} \des(\DS,\DT)  + \lambda_\posterior \,.\\[-5mm] &  \qedhere
\end{align*}
\end{proof}

The improvement of Theorem~\ref{thm:pacbayesdabound} over Theorem~\ref{thm:pbda}  relies on two main points.
On the one hand, our new result contains only the half of $\des(\DS,\DT)$. 
On the other hand, contrary to $\lambda_{\posterior,\posterior_{\mbox{\tiny $T$}}^*}$ of Eq.~\eqref{eq:pbda}, the term $\lambda_{\posterior}$ of Eq.~\eqref{eq:new}
does not depend anymore on the best $\posterior_T^*$ on the target domain. This implies that our new bound is not degenerated when the two distributions $\PS$ and $\PT$ are equal (or very close). Conversely, when $\PS = \PT$, the bound of Theorem~\ref{thm:pbda} gives
\begin{align*}
\RPT(\GQ) \leq  \RPT(\GQ) + \RPT(G_{\posterior_T^*}) + 2 \RDT(\GQ,G_{\posterior_T^*})\,, 
\end{align*}
which is at least $2\RPT(G_{\posterior_T^*})$. Moreover, the term $2 \RDT(\GQ,G_{\posterior_T^*})$ is greater than zero for any  $\posterior$  when the support of  $\posterior$ and  $\posterior_T^*$ in $\Hcal$ is constituted of at least two different classifiers. 

\subsection{A New PAC-Bayesian Bound}
Note that the improvements introduced by Theorem~\ref{thm:pacbayesdabound} do not change the form and the philosophy of the PAC-Bayesian theorems previously presented by Germain et al.~\cite{PBDA}. Indeed, following the same proof technique, we obtain the following PAC-Bayesian domain adaption bound.
\begin{theorem}
 \label{theo:pacbayesdabound_catoni_bis}
 For any domains $\PS$ and $\PT$ (resp. with marginals $\DS$ and $\DT$) over $X \times  Y$, any set of hypothesis  $\Hcal$,  any prior distribution $\prior$ over $\Hcal$, any $\delta \in (0,1]$, any real numbers $\alpha > 0$ and $c > 0$,  with a probability at least $1-\delta$ over the choice of $S \times  T  \sim (\PS \times  D_T)^m $, for every posterior distribution $\posterior$ on $\Hcal$, we have
 \begin{align*} 
\RPT(G_\posterior)  
\ \leq\  c'\, \RS(G_\posterior)  +  \alpha'\, \tfrac{1}{2} \des(S,T) + 
  \left( \frac{c'}{c} + \frac{\alpha'}{\alpha} \right)  \frac{\KL(\posterior\|\prior)+\ln\frac{3}{\delta}}{m}
   + \lambda_\posterior + \tfrac{1}{2} (\alpha'\!-\! 1)
   \,,
 \end{align*}
where 
$\lambda_\rho$ is defined by Eq.~\eqref{eq:lambda_rho}, 
and where \,
$\displaystyle c'\eqdef\frac{c}{1 -e^{-c}}$, \, and \, $\displaystyle \alpha'\eqdef \frac{2\alpha}{1 -e^{-2\alpha}}$\,.
\end{theorem}

\subsection{On the Estimation of the Unknown Term $\lambda_\rho$}

The next proposition gives an upper bound on the term $\lambda_\rho$ of Theorems~\ref{thm:pacbayesdabound} and~\ref{theo:pacbayesdabound_catoni_bis}.
\begin{proposition}
\label{prop:lambda}
Let $\Hcal$ be the hypothesis space. If we suppose that $\PS$ and $\PT$ share the same support, then
$$
\forall \posterior\mbox{ on } \Hcal,\ \lambda_\posterior\ \leq\  \sqrt{  \chi^2\big(\PT\|\PS\big) \, \ePS(G_\posterior,G_\posterior)}\,,
$$
where $\ePS(G_\posterior,G_\posterior)$ is the \emph{expected joint
  error} on the source distribution, as defined by Eq.~\eqref{eq:eP},
and $\chi^2\big(\PT\|\PS\big)$ is the \emph{chi-squared} divergence
between the target and the source distributions.
\end{proposition}
\begin{proof}
 Supposing that $\PT$ and $\PS$ have the same support, then we can
 upper bound $\lambda_\posterior$ using Cauchy-Schwarz inequality to
 obtain line 4 from line 3.
 \begin{align*} 
 \nonumber
 \lambda_\posterior
 &=
 \left|\esp{(h,h')\sim\posterior^2}\!\!\left[\esp{(\xbf,y) \sim \PT}\!\!\!\! \I[h(\xb)\neq y] \,\I[h'(\xb)\neq y]  -\!\! \esp{(\xbf,y) \sim \PS}\!\!\!\! \I[h(\xb)\neq y]  \,\I[h'(\xb)\neq y] \right]\right|\\
 &=
 \left|\esp{(h,h')\sim\posterior^2}\!\!\left[\esp{(\xbf,y) \sim \PS}
\!\! \frac{\PT(\xbf,y)}{\PS(\xbf,y)}
  \I[h(\xb)\neq y] \,\I[h'(\xb)\neq y]  -\!\! \esp{(\xbf,y) \sim \PS}\!\!\!\! \I[h(\xb)\neq y]  \,\I[h'(\xb)\neq y] \right]\right|\\
 &=
 \left|\esp{(h,h')\sim\posterior^2} \esp{(\xbf,y) \sim \PS}
\!\! \left(\frac{\PT(\xbf,y)}{\PS(\xbf,y)}-1\right)
  \I[h(\xb)\neq y] \,\I[h'(\xb)\neq y]  \right|\\
  &\leq 
  \sqrt{
  \esp{(\xbf,y) \sim \PS}
 \!\! \left(\frac{\PT(\xbf,y)}{\PS(\xbf,y)}-1\right)^2} 
 \times
\sqrt{ \esp{(h,h')\sim\posterior^2} \esp{(\xbf,y) \sim \PS}
   \left( \I[h(\xb)\neq y] \,\I[h'(\xb)\neq y] \right)^2 } \\
    &\leq 
    \sqrt{
    \esp{(\xbf,y) \sim \PS}
   \!\! \left(\frac{\PT(\xbf,y)}{\PS(\xbf,y)}-1\right)^2 
   \times
   \esp{(h,h')\sim\posterior^2} \esp{(\xbf,y) \sim \PS}
      \I[h(\xb)\neq y] \,\I[h'(\xb)\neq y] } \\
   &=   
       \sqrt{
       \esp{(\xbf,y) \sim \PS}
      \!\! \left(\frac{\PT(\xbf,y)}{\PS(\xbf,y)}-1\right)^2 
      \times
       \ePS(G_\posterior,G_\posterior)}
  \ = \ \sqrt{  \chi^2\big(\PT\|\PS\big) \, \ePS(G_\posterior,G_\posterior)}\,.
  \\[-6mm] & \qedhere
 \end{align*}
\end{proof}
This result indicates that $\lambda_\posterior$ can be controlled by the term
$\ePS$, which can be estimated from samples, and the chi-squared
divergence between the two distributions that we could try to estimate
in an unsupervised way or, maybe more appropriately, use as a constant
to tune, expressing a tradeoff between the two distributions. This
opens the door to derive new learning algorithms for domain adaptation
with the hope of controlling in part some negative transfer.

\bibliography{biblio}
\bibliographystyle{unsrt}

\end{document}

%% file: entete.tex
\newcommand{\RDS}{R_{\DS}}

\newcommand{\RDT}{R_{\DT}}
\newcommand{\RD}{R_{D}}

\usepackage{amsmath}
\usepackage{amssymb}
\usepackage{enumerate}
\usepackage{color}
\usepackage{mathtools}
\usepackage{xspace}

\usepackage{upgreek}

\usepackage{hyperref}

\definecolor{mydarkblue}{rgb}{0,0.08,0.45} 
\hypersetup{
  colorlinks   = true, 
  urlcolor     = mydarkblue, 
  linkcolor    = mydarkblue, 
  filecolor    = mydarkblue,
  citecolor    = mydarkblue 
}

\makeatletter
\def\captionof#1#2{{\def\@captype{#1}#2}}
\makeatother

\newcommand{\hdh}{\Hcal\!\Delta\!\Hcal}

\newcommand{\xbf}{\ensuremath{\mathbf{x}}}

\newcommand{\PS}{P_S}
\newcommand{\PT}{P_T}
\newcommand{\mt}{m_t}

\newcommand{\ms}{m_S}
\newcommand{\xbfs}{\xbf_s}
\newcommand{\xbft}{\xbf_t}

\newcommand{\ys}{y_s}

\newcommand{\DS}{D_S}
\newcommand{\DT}{D_T}

\newcommand{\RPS}{R_{\PS}}

\newcommand{\RPT}{R_{\PT}}

\newcommand{\RS}{R_{S}}

\newcommand{\Hcal}{\ensuremath{\mathcal{H}}}

\newcommand{\BQ}{B_\posterior}
\newcommand{\GQ}{G_\posterior}

\newcommand{\posterior}{\rho}
\newcommand{\prior}{\pi}

\newcommand{\sign}{\operatorname{sign}}

\newcommand{\des}{\operatorname{dis}_\posterior}

\newcommand{\R}{\mathbb{R}}
\newcommand{\I}{\mathbf{I}}

\definecolor{vertee}{rgb}{0,.5,0}

\newcommand{\eqdef}{\overset{{\mbox{\rm\tiny def}}}{=}}

\newcommand{\esp}[1]{
    \underset{#1}{\mathrm{\bf E}}\
}

\newcommand{\argmindevant}[1]{
    {\mathrm{argmin}}_{#1}\
}

\newcommand{\xb}{\mathbf{x}}

\newcommand{\KL}{{\rm KL}}

\newcommand{\ePS}{{e_{\PS}}}
 \newcommand{\ePT}{{e_{\PT}}} 
 \newcommand{\eP}{{e_{P}}}